\documentclass[conference,usenames,12pt,a4paper]{IEEEtran}
\pdfoutput=1   

\usepackage[T1]{fontenc}
\usepackage{ifpdf}
\usepackage{flushend}  
\usepackage{ulem}
\usepackage{alphalph}
\usepackage[unicode=true, bookmarks=true,bookmarksnumbered=false,
            bookmarksopen=false,hyperfootnotes=false,breaklinks=true, pdfborder={0 0 1},backref=false,colorlinks=true]{hyperref}
\usepackage{amsmath}
\usepackage{amssymb}
\usepackage{amsthm}
\newtheorem{theorem}{Theorem}
\newtheorem{proposition}[theorem]{Proposition}
\newtheorem{lemma}[theorem]{Lemma}

\usepackage{multirow}
\usepackage[dvipsnames]{xcolor}
\usepackage{graphicx}  
\usepackage[font=footnotesize,caption=false]{subfig}
\usepackage{caption} 

\usepackage{booktabs, makecell, tabularx} 
\usepackage{rotating}
\usepackage[export]{adjustbox}

\usepackage{tablefootnote}

\usepackage{url}
\usepackage{soul}

\usepackage{orcidlink}

\newcolumntype{L}{>{\centering\arraybackslash}m{1cm}} 
\makeatletter
\newcommand{\printfnsymbol}[1]{%
  \textsuperscript{\@fnsymbol{#1}}%
}
\makeatother

\addtocounter{secnumdepth}{1}

\IEEEoverridecommandlockouts 

\begin{document}
\title{On The Relationship between Visual Anomaly-free and Anomalous
Representations$^*$\thanks{$^*$To appear in ECCV'24 Workshop on Industrial
Inspections}}

\author{
\IEEEauthorblockN{Riya Sadrani\textsuperscript{\textsection}}
\IEEEauthorblockA{\textit{TCS Research} \\
riya.sadrani@tcs.com}
\and
\IEEEauthorblockN{Hrishikesh Sharma,~\IEEEmembership{Member,~IEEE,}\orcidlink{0000-0001-9647-0661}\textsuperscript{\textsection}}
\IEEEauthorblockA{\textit{TCS Research} \\
hrishikesh.sharma@tcs.com}
\and
\IEEEauthorblockN{Ayush Bachan}
\IEEEauthorblockA{\textit{TCS Research} \\
ayush.bachan@tcs.com}
}

\maketitle              
\begingroup\renewcommand\thefootnote{\textsection}
\footnotetext{Equal contribution}
\endgroup

\begin{abstract}
Anomaly Detection is an important problem within computer vision, having
variety of real-life applications. Yet, the current set of solutions to
this problem entail known, systematic shortcomings. Specifically,
contemporary surface Anomaly Detection task assumes the presence of
multiple specific anomaly classes e.g. cracks, rusting etc., unlike
one-class classification model of past. However, building a deep learning
model in such setup remains a challenge because anomalies arise rarely, and
hence anomaly samples are quite scarce. Transfer learning has been a
preferred paradigm in such situations. But the typical source domains with
large dataset sizes e.g. ImageNet,  JFT-300M, LAION-2B do not correlate
well with the domain of surfaces and materials, an important premise of
transfer learning. In this paper, we make an important hypothesis and show,
by \textbf{exhaustive} experimentation, that the space of anomaly-free
visual patterns of the normal samples correlates well with each of the
various spaces of anomalous patterns of the class-specific anomaly samples.
The first results of using this hypothesis in transfer learning have indeed
been quite encouraging. We expect that finding such a simple closeby domain
that readily entails large number of samples, and which also oftentimes
shows interclass separability though with narrow margins, will be a useful
discovery. Especially, it is expected to improve domain adaptation for
anomaly detection, and few-shot learning for anomaly detection, making
in-the-wild anomaly detection realistically possible in future.
\end{abstract}
\section{Introduction}
\label{intro_sec}
Anomaly detection (AD) is an longstanding computer vision problem, which has gained renewed interest due to major advances in deep learning (DL) techniques \cite{ad_review_pap}. As is known, usage of DL has consitently shown performance improvement on allied tasks such as object detection and segmentation. Not only newer, more complex and diverse \textbf{AD} datasets have been released in recent times \cite{visa_pap,mvtec_ad_pap,neu_det_pap,btad_pap,codebrim_pap,sewer_ws_pap}, but also there has been a steady increase in research, especially in computer vision, on AD task modeling \cite{cfa_pap,phy_crack_pap,ad_dist_shift_pap,transfer_al_pap,blend_adsynth_pap}.  The interest is also fueled by equal interest arising from industries, especially with manufacturing firms adopting Industry 4.0 standard that entails \textit{intelligent automation} of industry processes such as inspections.

The traditional task modeling of AD problem has been to consider it as a \textit{one-class classification} (OCC) problem \cite{deep_occ_pap,ad_review_pap}. This consideration is also reflected in the manner some of the recent, more real-life-like datasets continue to be organized \cite{mvtec_ad_pap,btad_pap}: the training subset contains exclusively of normal samples, forming one class. However, this thinking is now being challenged. Against the perpetuated OCC modeling, inspired by the success of supervised learning within DL, recent datasets such as \cite{codebrim_pap,neu_det_pap,sewer_ws_pap} assume presence of \textbf{multiple} visual anomaly classes, that need to be discriminated. From a causal perspective, different anomalies in industrial context arise from various different types of malfunctions, and it becomes important to identify the malfunction, to fix operational problems.

One problem with this anomaly detection context arises from the strict requirements of deep learning. Especially in supervised learning and even in self-supervised learning settings, deep learning models require a lot of class-specific data to train upon. It is well-known in AD area, that other than a few anomaly classes such as cracks that arise more frequently, the other anomalies occur only rarely.  Hence in general, there are not enough samples, even of the order of hundreds, for most of the anomalies, in the best of the contemporary datasets.

In lack of enough samples of various anomaly classes, whether labeled or not, one line of AD modeling relies on pretrained models from other domains, which can then be fine-tuned on anomalous data domains. For example, Imagenet-pretrained models have been used \textit{quite often} in deep AD, some examples being \cite{att_al_pap,cutpaste_ad_pap,spade_ad_pap,differnet_ad_pap,reg_ad_pap, rep_occ_pap,panda_ad_pap}. However, as is understood, the domain of natural images and the domain of material surfaces have a \textbf{large} domain gap. This gap results in many problems, such as poor localization till date \cite{saa_pap}, incompatible knowledge transfer \cite{fsig_incomp_pap}, features remaining uninformative even after fine-tuning \cite{ms_ad_pap}, to name a few. Our own past experiments with anomaly localization also yielded similar insights. Even the recent \textit{multiclass} AD datasets with labels are still suggestive of a train:test split \cite{codebrim_pap,sewer_ws_pap}, which avoids domain shift between train and test data. Based on these gaps, we pose the following important question: \textit{does there exist a more intuitive domain that has lesser domain gap with various visual anomalous patterns and classes?}

To answer this question, we make a hypothesis: \textit{the space of anomaly-free visual patterns of the normal samples correlates\footnote{We use ``correlation" in the English sense, to refer to linear as well as non-linear class relationships.} well with each of the various spaces of anomalous visual patterns of the class-specific anomaly samples}. To verify this hypothesis, we work with latent representations of class-specific visually anomalous regions and normal regions. We vary the anomaly datasets to have tens of anomaly classes, employ tens of popular backbones of heterogeneous types, employ various image-space and latent-space metrics and conduct exhaustive experiments to verify the hypothesis both qualitatively and quantitatively. Other than a few outliers, we find this hypothesis to be generally true. The hypothesis is important since close dependence among such patterns opens up gates for improved domain adaptation along with anomaly detection, enabling few-shot AD models via transfer learning etc. 
In summary, our contribution is as follows:

\begin{itemize}
\item  We conduct exhaustive empirical examination of statistical dependence between anomaly-free and anomalous visual patterns on several real-world datasets, using several backbones and several metrics.
\item We first-time\footnote{To the best of authors' knowledge} establish the hypothesis about close class relationships between several anomaly-free and anomalous visual patterns.
\item We employ the hypothesis and show promising results, surpassing a recent baseline.
\end{itemize}

The rest of the paper is organized as follows. We present related work in next section, followed by details of the exhaustive experimental setup for hypothesis testing in section \ref{exp_sec}. The analyses of various results in section \ref{res_sec} is presented in section \ref{anal_sec}, hypothesis employment in section \ref{tl_sec}, before we conclude the paper.

\begin{figure*}[t]
\begin{center}
\subfloat[Anomalous Image]{\label{crop_1}\includegraphics[trim=150 80 140 100, clip,scale=0.09]{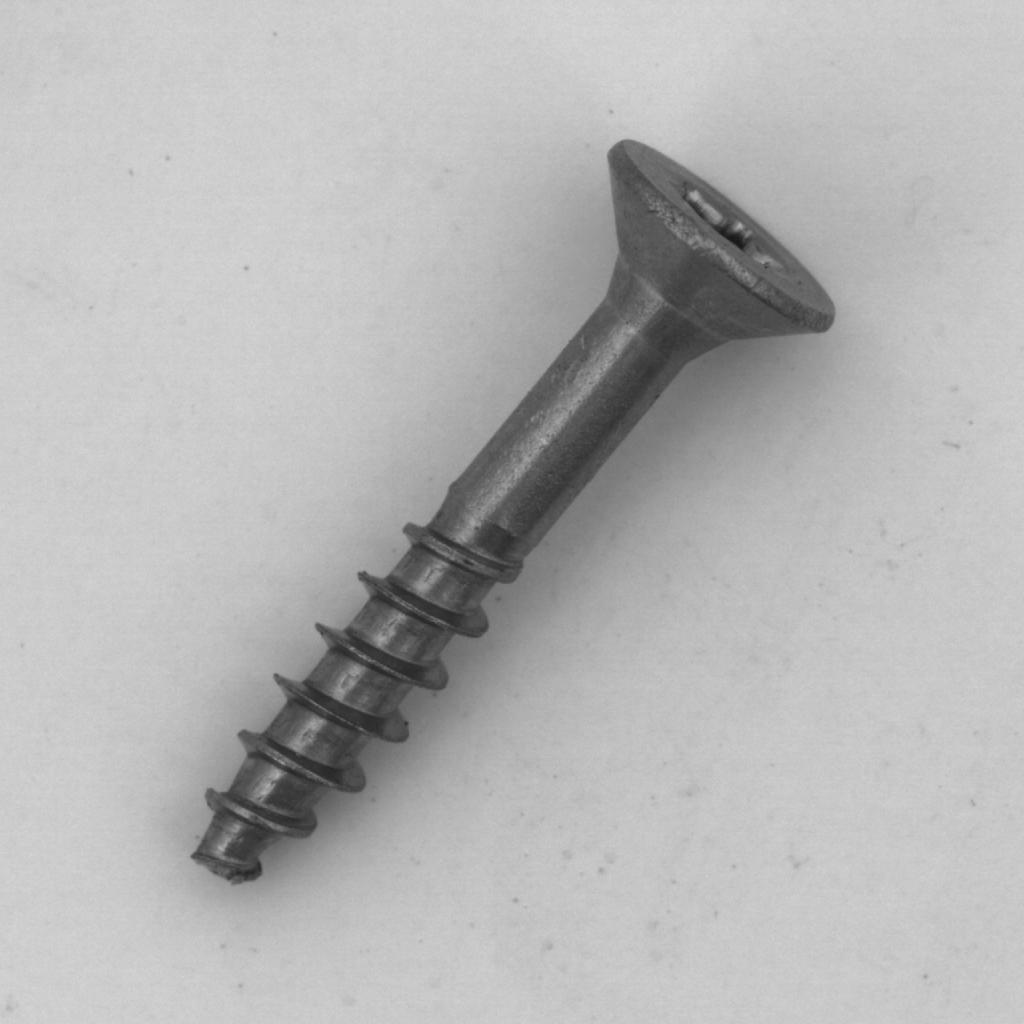}}
\qquad\qquad\qquad\qquad
\subfloat[Cropped Defect Sample]{\label{crop_2}\includegraphics[scale=.73]{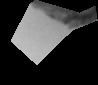}}
\qquad\qquad\qquad\qquad
\subfloat[Cropped Defectless Sample]{\label{crop_3}\includegraphics[scale=.73]{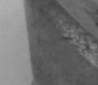}}
    \caption{Example Data Preparation from One Image}
\label{crop_fig}
\end{center}
\end{figure*}

\section{Related Work}
\label{bg_sec}
In recent past, there have been studies about class relationships. However, as far as we know, there has been no study that establishes the correlation between anomaly-free and anomalous visual classes. Hence we summarize the recent works on \textit{somewhat related} problems here.

Data-driven studies of various non-obvious relationships has been happening in machine learning community for sometime now. One may look at \cite{tcvae_pap},\cite{rel_conflict_pap},\cite{exp_pred_pap} for some of such recent studies. While there have been studies at domain-level and instance-level relationships as well, class-level relationships have been studied relatively more. Class-level \textit{semantic} relationship study drives the study of scene graphs and caption generation. Class relationships modeled as \textit{statistically dependent patterns} has been recently used in \cite{class_rel_pap} and \cite{rel_kd_pap}.

The usage of distances for establishing transferability between dataset domains has been studied in past. Such computation indirectly entails looking at class relationships and class separability. Few such recent works can be found in \cite{geo_ddist_pap}, \cite{bhat_sep_pap}.

Fueled by publication of newer defect and damage datasets
\cite{visa_pap,mvtec_ad_pap,neu_det_pap,btad_pap,codebrim_pap,sewer_ws_pap},
Anomaly detection and localization has become a mainstream computer vision
problem, with tens of recent interesting works
\cite{cfa_pap,phy_crack_pap,att_al_pap,reg_ad_pap,cutpaste_ad_pap,spade_ad_pap,differnet_ad_pap,rep_occ_pap,panda_ad_pap,saa_pap,ms_ad_pap}.
We refer the reader to \cite{vad_survey_pap} for an in-depth understanding
and comparison of the most important works on this problem. The tenets and
insights about the problem can be understood from \cite{ad_review_pap}.
Though many of these works bring out the data-scarcity scenario within AD
problem, not many of these works openly point out to the underlying trouble
of large domain shift/gap. A few of the works that allude to domain shift
and consequent domain gap are \cite{panda_ad_pap},
\cite{ad_dist_shift_pap}, \cite{transfer_al_pap}. Some of the applied works
which faintly assume statistical dependence between anomalous and
non-anomalous image regions, without formally mentioning or establishing the same, are \cite{blend_adsynth_pap}, \cite{blend_latent_pap}, \cite{fsig_incomp_pap}, \cite{fsig_ewc_pap}. In the remaining paper, we explicitly bring out details of such relationship.

\begin{table*}[t]
\caption{Information about Backbones}
\label{bb_tab}
\begin{center}
\begin{tabular}{|p{.25\textwidth}|p{.66\textwidth}|} \hline \hline
\multicolumn{1}{|c|}{\textbf{Architecture Category}} & \multicolumn{1}{c|}{\textbf{Included Backbones}} \\ \hline \hline
Transformers     &  Poolformerv1, ViT, FlexiViT, BEiTv1, BEiTv2, DEiTv1, DEiTv3, EVAv1, EVAv2, CAiT, LV-ViT, TnT, DeepViT, CSWin, XCiT, T2TViT, SWiNv1, SWiNv2, DINOv2, CLIP, SWAG \\ \hline
Traditional CNNs     &  ResNet18, BiT, RepVGG, MobileOne, EfficientNet, FocalNet, DenseNet, HRNet, NFNet, ShuffleNetv1, ShuffleNetv2, Wide-Resnet, XCeption \\ \hline
Large-kernel CNNs & ConvNeXT, RepLKNet, SLaKNet, VANet \\ \hline
Hybrids & FastViT, EfficientFormerv1, Efficientformerv2, PVTv1, PVTv2, TWINS, VOLO, PiT, BoTNet, MobileViT, Conformer, CoAT, CMT \\ \hline
Self-supervised & SEER, ConvNextv2 \\ \hline
MLPs & MLP-Mixer \\ \hline
\end{tabular}
\end{center}
\end{table*}

\section{Experimental Setup for Establishing Correlation}
\label{exp_sec}
The regularity in \textit{normal} objects of a particular class imaged in real-world, noisy environment is generally modeled using stochastic patterns. The pattern could be based on shape, color, texture or a mix of all, and it can also be a latent/representation space pattern as well \cite{tp_bookchap}. Anomalous objects or regions have been traditionally considered to be outliers from the probability distribution that underlies the normal pattern \cite{rep_occ_pap}. However, as pointed out in section \ref{intro_sec},
recent works discriminate various anomaly patterns, and consider them as separate anomaly classes. For example, cracks, dents, bends are some well-known anomaly classes, discriminable by even a human eye, with annotations also available from some datasets.

To recall, the research question that we posed in section \ref{intro_sec} was that \textit{does there exist a more intuitive domain that has lesser domain gap with various visual anomalous patterns and classes?}. More formally, we seek to statistically establish (or disprove) that (defect-)class-specific surface defect patterns are related to the patterns of the normal surfaces on which they appear. The relationship can be of the form of statistical dependence between the \textbf{representations} of corresponding samples/regions, in which case one can think about employing correlation or covariance measures as 2$^{nd}$-order\footnote{Statistical independence entails zero correlation, not vice-versa \cite{murphy_ml_book}.} \textit{population-level} relational statistics over multivariate distributions. However, as is well-known, these statistics only model linear relationships between data \cite{pr_bishop_book}, while the succinct representation space of images is essentially non-linear manifolds \cite{dl_book}, which is \textit{not} modeled by these \cite{murphy_ml_book}. Indeed, covariance entails computing an inner product which is meant for algebraic spaces, but has no meaning over topological manifolds (the corresponding relational concept is homeomorphism). To further complicate matters, one may observe that the manifold representation spaces are known to be very high-dimensional spaces with implicit \textit{continuous}, empirical distributions, thus ruling out usage of Chi-squared test as well, for formally verifying the hypothesis.

To get around this limitation, we resort to measurements of divergences and distances. As mentioned in section \ref{bg_sec}, they have been used to indirectly characterize class relationships in past, albeit in a totally different context. These measures define the fitment/similarity of two probability distribution functions (pdfs), whether closed-form or empirical, and hence indirectly characterize the relationship and dependence between two patterns/classes. These measures are related to a broader information-theoretic measure called \textit{mutual information}; we capture that as well. The choice of latent representation space to make such measures is inspired by \cite{dom_shift_pap}, wherein they successfully use image latent space to measure domain gaps.

The sequence of steps for establishing the correlation are depicted in
Fig.~\ref{fc_fig}.

\begin{figure*}[h]
\begin{center}
\includegraphics[scale=0.5]{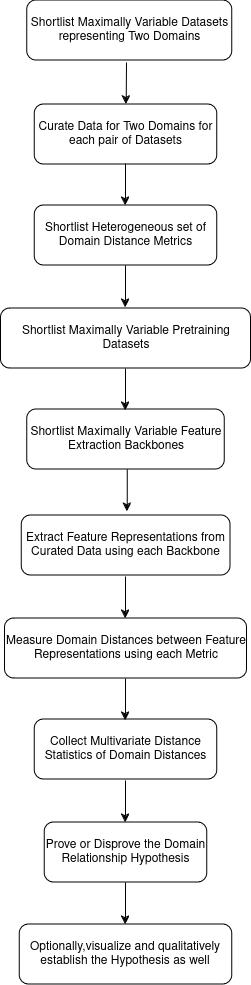}
\caption{Sequence of Steps for Establishing Correlation}
\label{fc_fig}
\end{center}
\end{figure*}

\subsection{Datasets Considered and their Preparation}
\label{dataprep_sec}
There has been a manyfold increase in the number of realistic, complex AD datasets in recent times. Some of such recent datasets include \cite{codebrim_pap,neu_det_pap,sewer_ws_pap,btad_pap,mvtec_ad_pap,visa_pap}. For a more detailed summary of AD datasets, one may refer to \cite{vad_survey_pap}. Given that the traditional AD modeling has been OCC problem, while the recent works consider it as a multiclass problem, we shortlist datasets of \textbf{both} nature. We hence consider MVTec-AD, NEU-DET, BTAD and CODEBRIM for our experiments. Across these datasets, we comprehensively cover \textbf{76} classes of different anomalies. CODEBRIM dataset, for which we report \textbf{only} the distances, is \textit{different} from others in the sense that it is collected during (outdoor) infrastructure inspection. Yet, we demonstrate that our hypothesis is generic enough for all kinds of surfaces, ambient conditions and defect types.

Each of these datasets entail images that are of normal objects, and images that have anomalies in specific regions. The name/class of each of the anomalous regions is also known from the dataset, via label information. Often, the objects being imaged have trivial background, e.g. a monotone gray tray in which a compressed/damaged pill is kept. Hence each \textit{class-specific defect sample} is created by best-fit\footnote{Care has been taken so that each sample is \textbf{pure}: it has no information contamination from any other pattern.} rectangular cropping of anomalous regions, whose contour is known via annotation. A similar crop of \textbf{same box size} at random location in object \textit{foreground} of the \textbf{same} image leads to generation of the corresponding, \textbf{paired} \textit{defectless sample}. The paired cropping is done using the same image, since it fixes the environmental conditions e.g. lighting, camera point of view etc. for both the paired samples (see Fig~\ref{crop_fig}). Notationally, $\forall$\textbf{r}: region within $\forall$\textbf{i}: image having $\forall$\textbf{c}: anomaly class having multiple instances \textbf{r} over $\forall$ $\mathbf{o_c}$: surface type on which particular anomaly arises on an object type within $\forall$\textbf{s}: dataset name, we prepare sample sets $\mathbf{D_{r,i,c,s}}$ and $\mathbf{N_{r,i,o_{c},s}}$. Samples in each set are pure and matched i.e. reshaped to same shape and size over black/information-less background, before their usage.

\subsection{Backbones Considered}
One of the main reasons of success of DL technology has been usage of \textit{transfer learning}. Transfer learning has been possible due to advent of \textbf{backbone networks}: encoders that are generalized enough to create a feature space representation, reusable across a diversity of tasks and datasets. There has been a longstanding race to create backbone networks, and newer backbones are arising frequently even now. For an excellent survey on these \textit{few hundreds of} backbones, one is referred to \cite{bbone_battle_pap}. We choose representatives from \textbf{each} of the \textbf{6} important \textit{types} of backbones, including normal CNN backbones, very-large-kernel CNN backbones and transformer backbones. The shortlisting of backbones has been also done by keeping in mind their popularity as well as their lifespan. Overall, we use \textbf{54} backbones, each of which was carefully chosen to be highly popular, spanning over last ten years.

\subsection{Measures Considered}

As explained earlier, we work mostly with divergences and distances between the two domains of representations: $\mathrm{D_{r,i,c,s}}$ and $\mathrm{N_{r,i,o_{c},s}}$. The representation of each sample is prepared by passing it through each of the backbones shortlisted in Table \ref{bb_tab}. Following \cite{21corr_pap} and \cite{found_ml_book}, we chose \textbf{JS Divergence}, \textbf{Mahalanobis Distance} and \textbf{Wasserstein Distance} as our required measures. Note that Mahalanobis Distance when calculated between two dataset level p.d.f.s is a form of \textit{Bregman Divergence} \cite{found_ml_book}. All of these measures are aggregates of sample pair level measures i.e. a sort of expectation computed over the two p.d.f.s.

As a secondary investigation, we also see if and how the p.d.f.s are correlated in \textit{pixel} space. All other metrics being a form of mutual information, we employ and additionally report a similar metric, \textbf{Regional Mutual Information}(RMI) metric in pixel space \cite{rmi_ss_pap,rmi_sim_pap}. In past, this metric has popularly been used in cross-modal (medical) image registration tasks.

\subsection{Visualizations}
We also look a look at \textbf{UMAP} visualizations \cite{umap_pap} to qualitatively look at the dataset nearness, given a particular backbone. Across backbones, such an exercise at times also reveals whether the geometric configuration of two clusters corresponding to two patterns $\mathrm{D_{r,i,c,s}}$ and $\mathrm{N_{r,i,o_{c},s}}$ remains consistent due to correlation i.e. similar degree of nearness, even when cluster sizes and locations vary due to obvious reasons.

\section{Experimental Results}
\label{res_sec}
The results on representation relations need to be seen from the prism of
distances considered. To recall, within each of the datasets considered,
there are $\geq$ 1 object classes i.e. foreground surfaces. For each
foreground surface, there are specific anomaly classes(defects) that can
arise. In Figs.~\ref{ds_1}, \ref{ds_2}, \ref{ds_3} and \ref{ds_4}, we plot the variation of each of the \textbf{3} metrics across the anomaly classes. The metrics has been \textit{averaged} across various backbones that were considered. We also present certain \textbf{ablation} results in Tables \ref{bb_js_ab_tab}, \ref{bb_mh_ab_tab} and \ref{bb_ws_ab_tab}.

\begin{figure*}[t]
\begin{center}
\subfloat[NEU Dataset Metric]{\label{ds_1}\includegraphics[trim=100 60 100 80, clip, width=.48\textwidth]{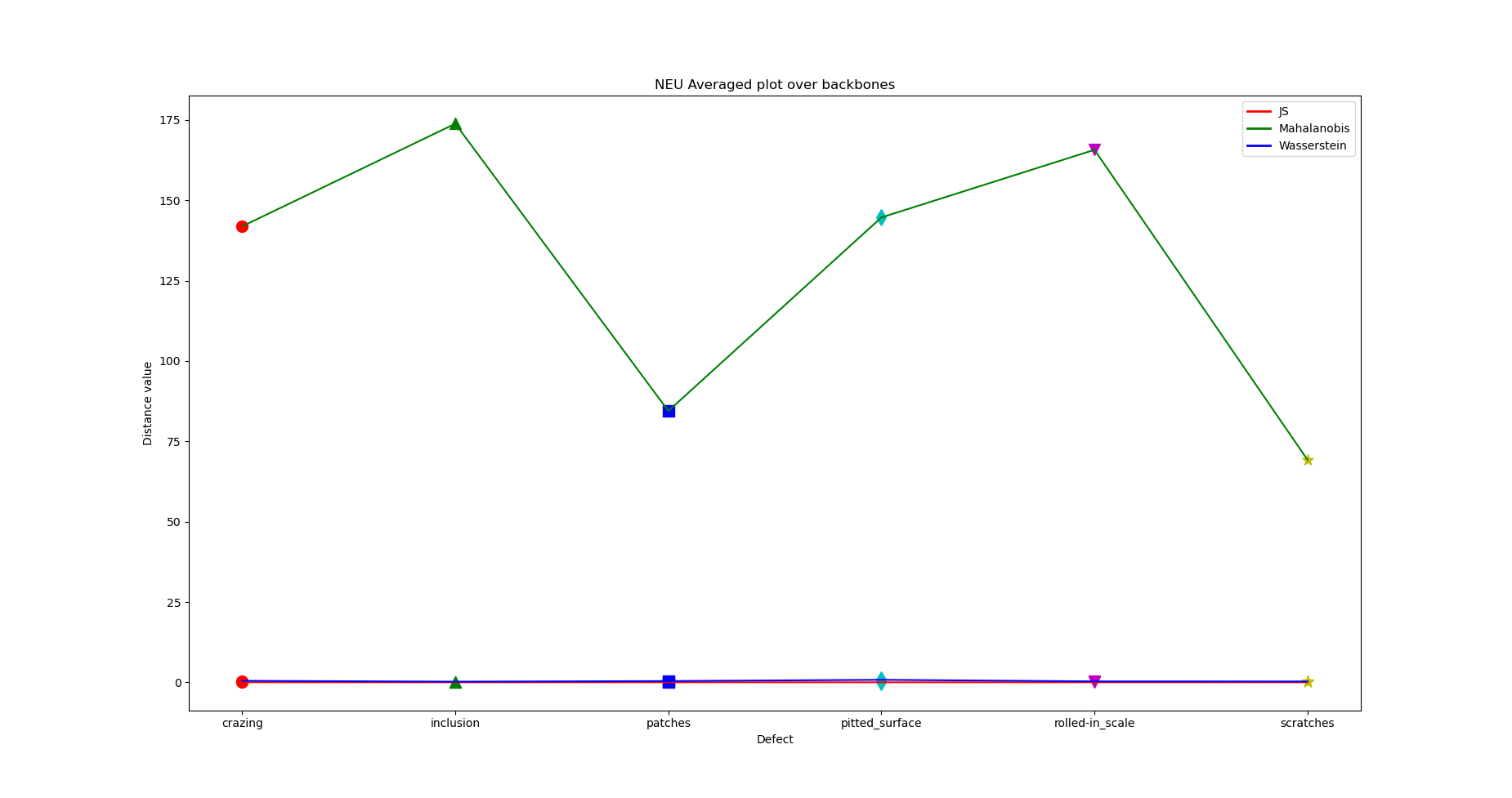}}
\quad
\subfloat[MVTec Dataset Metric]{\label{ds_2}\includegraphics[trim=100 60 100 80, clip,width=.48\textwidth]{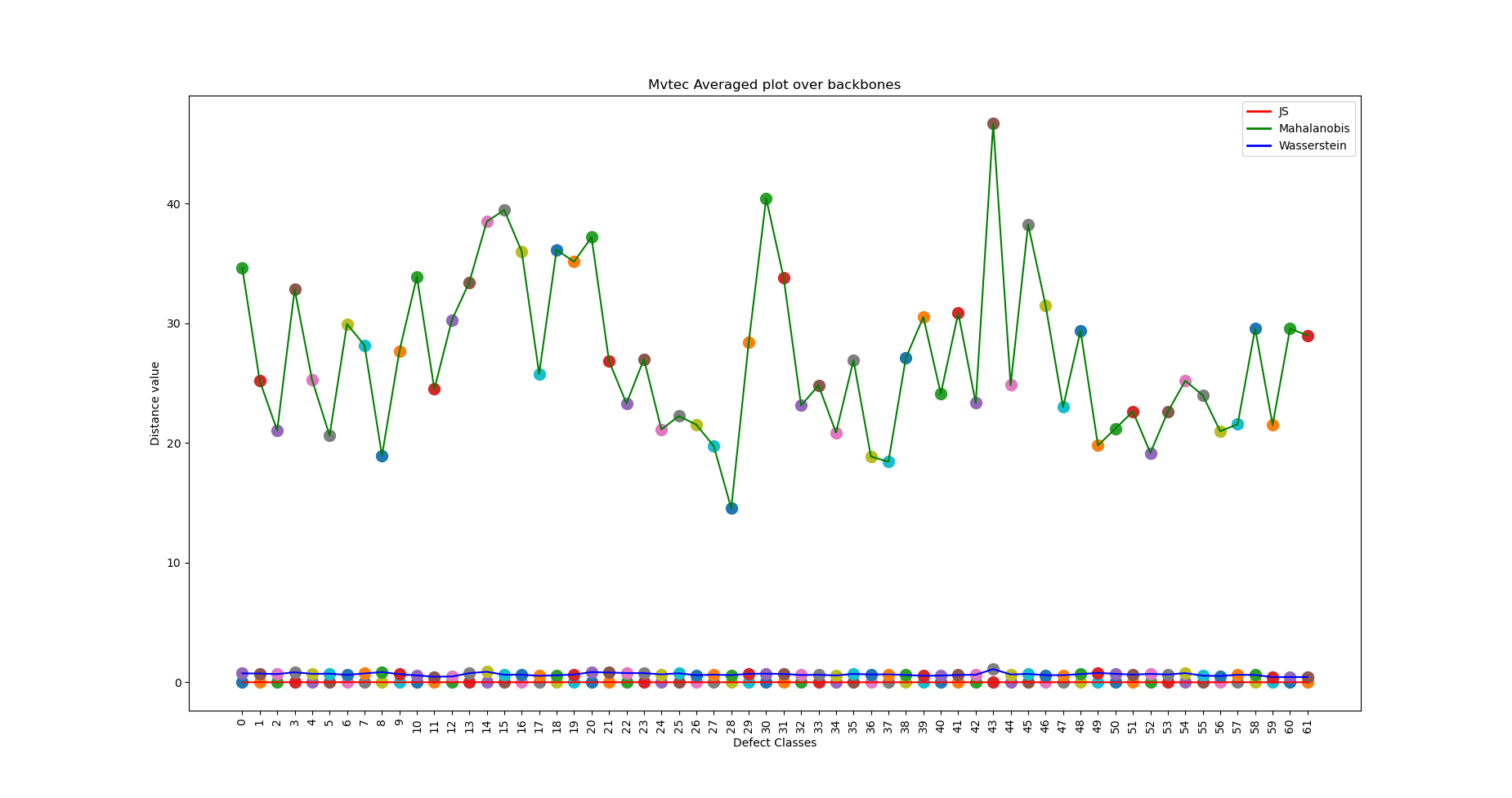}}
\quad
\subfloat[BTAD Dataset Metric]{\label{ds_3}\includegraphics[trim=100 60 100 80, clip,width=.48\textwidth]{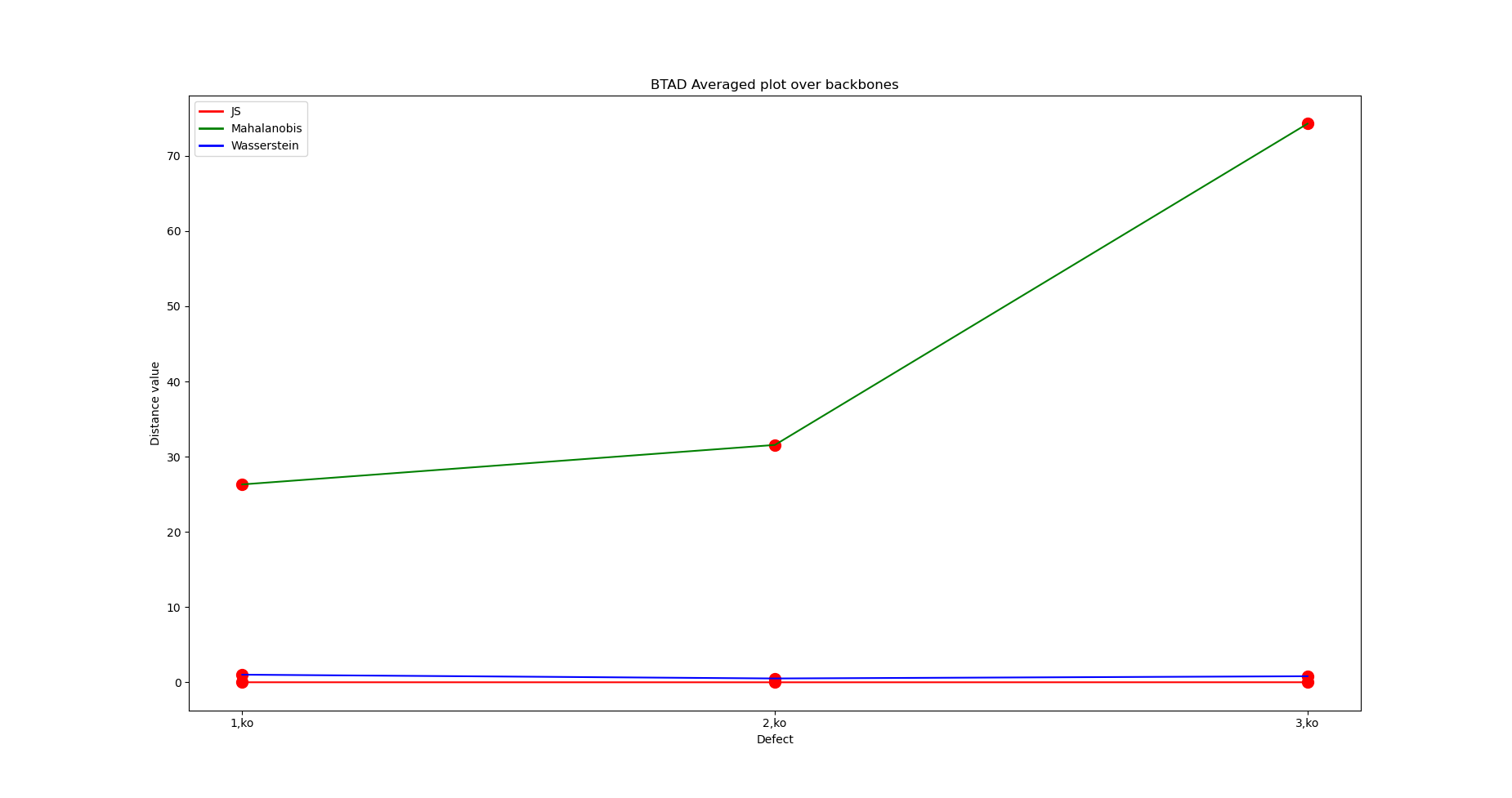}}
\quad
\subfloat[Codebrim Dataset Metric]{\label{ds_4}\includegraphics[trim=0 0 0 0, clip,width=.48\textwidth]{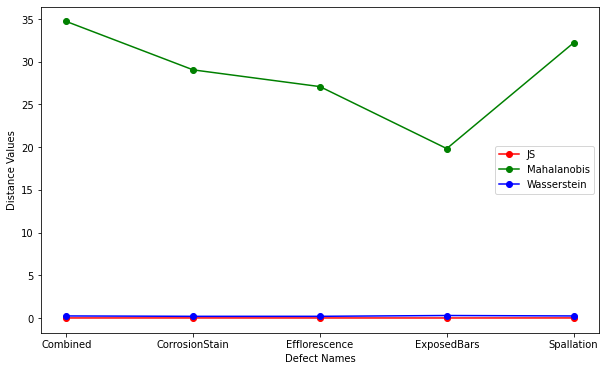}}
\caption{Domain Separation between Anomaly-free and Anomalous Domains. Zoom for better view.}
\label{metric_fig}
\end{center}
\end{figure*}

For understanding RMI results, for each of the datasets, we provide separate scatter plots depicting RMI values for each paired $\langle$anomaly type - anomaly-free surface type$\rangle$. They are shown in Figs.~\ref{rv_1}, \ref{rv_2} and \ref{rv_3}.

The UMAP visualizations of inter-cluster relations for each of the \textbf{3} datasets, for randomly chosen anomaly classes.

\section{Analysis}
\label{anal_sec}
Before we analyze our findings, one needs to understand the lower and upper
bounds between which each of the metrics reside, wherever feasible.

\subsection{Theoretical Ranges of Each Measure}
There are four quantitative measures that we have employed: JS Divergence,
Mahalanobis Distance, Wasserstein Distance and Region Mutual Information.
For the former two, it is possible to put down an upper and a lower bound.
We provide the proof of the bounds for these two measures below. For the
latter two, it is not possible to put down universal bounds. Nevertheless,
we will make a statement about bounds on these latter two measures, and
also sketch out our analysis of whatever attempts have been made to put a
bound on them.

\begin{lemma}
\label{lm1}
The value of scalar KL divergence measure is always non-negative i.e. $>$ 0.
\end{lemma}
\begin{proof}
The KL divergence between p.d.f. \( P\) and p.d.f. \( Q\) is defined as: 
\begin{equation}
\label{eq1}
D_{\text{KL}}(P \| Q) = \sum_{x} P(x) \log \left( \frac{P(x)}{Q(x)} \right)
\end{equation}
where the sum is taken over all paired samples \textbf{x} prepared earlier. By definition, the information entrop of \(P\) and \(Q\) is given by:
\[
H(P) = -\sum_{x} P(x) \log P(x)
\]
which, by Gibb's inequality, is $\mathbf{\leq}$ it's cross-entropy with \textbf{any} other distribution \( Q\):
\begin{eqnarray*}
H(P, Q) & = & -\sum_{x} P(x) \log Q(x) \\
H(P) & \leq & H(P, Q)
\end{eqnarray*}
On reorganizing eq.~\ref{eq1}, we get,
\[
D_{\text{KL}}(P \| Q) ~=~ H(P, Q) - H(P)
\]
Thus, KL Divergence will always be \textit{non-negative} i.e. $\mathbf{\geq 0}$. 
\end{proof}

\begin{figure*}[t]
\begin{center}
\subfloat[NEU Dataset]{\label{rv_1}\includegraphics[trim=5 5 40 34, clip,width=.31\textwidth]{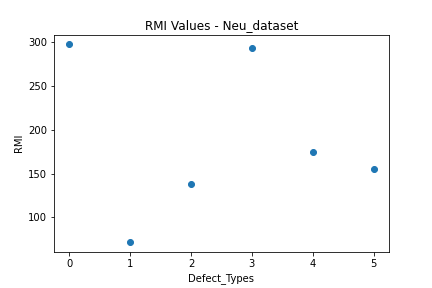}}
\quad
\subfloat[MVTec Dataset]{\label{rv_2}\includegraphics[trim=5 5 40 34, clip,width=.31\textwidth]{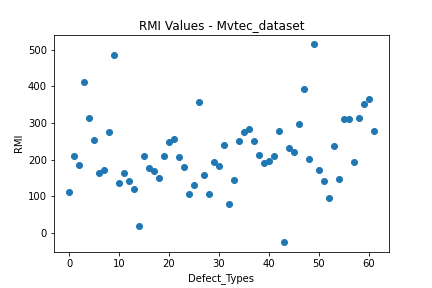}}
\quad
\subfloat[BTAD Dataset]{\label{rv_3}\includegraphics[trim=5 5 40 34, clip,width=.31\textwidth]{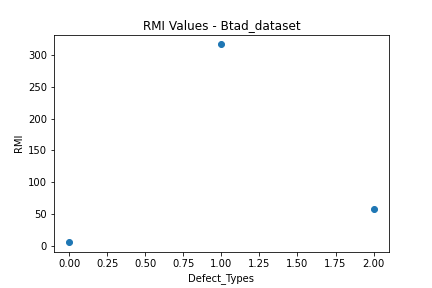}}
\caption{Scatter Plot for RMI Values for various Anomaly classes within each Dataset}
\label{rmi_fig}
\end{center}
\end{figure*}

\begin{theorem}
\label{js_th}
The range of values that the scalar JS divergence measure takes is between 0 and 1.
\end{theorem}
\begin{proof}
JS Divergence is computed as \cite{murphy_ml_book}
{\small
\begin{eqnarray}
\label{jsd_eqn}
    JSD(P\Vert Q) &=& \frac{1}{2}KLD(P\Vert M) + \frac{1}{2} KLD(Q\Vert M) \\
    \mbox{where~} M &=& \frac{1}{2}(P+Q)~~\mbox{(mixture distribution)}
\end{eqnarray}
}
From Lemma \ref{lm1}, both R.H.S. terms of (\ref{jsd_eqn}) are lowered bounded by 0, and hence JSD itself is lower bounded by 0. For the upper bound, we refer to Theorem 4 of \cite{shannon_div_pap}. Equation 4.3 therein states that
\begin{equation}
\label{p_e_eqn}
P_e(P,Q) ~\leq~ \frac{1}{2}\left( H(\pi_1, \pi_2) -JS(P\Vert Q)\right)
\end{equation}
for any $\pi_1$ and $\pi_2$ s.t. $\pi_1,~\pi_2~ \geq 0$, $\pi_1+\pi_2=1$, $H(\pi_1,\pi_2)$ is Shannon entropy function($\leq$ 1) and $P_e$ is Bayes (\textbf{$\geq 0$, non-negative}) probability of error. Reorganizing (\ref{p_e_eqn}), we have 

{\small
\begin{equation*}
JS(P\Vert Q) \leq H(\pi_1, \pi_2) - 2\cdot P_e(P,Q) \leq H(\pi_1, \pi_2) \leq 1
\end{equation*}
}
\end{proof}
\begin{theorem}
\label{mb_th}
The range of values that the scalar Mahalanobis distance measure takes is between 0 and \(\frac{(n - 1)p}{n} \) or $\frac{(n-1)^2}{n}$, depending on whether \textbf{n > (p+1)} or otherwise. Here, \textbf{n} is the number of paired samples and \textbf{p} is the length of each feature vector which represents each sample (given a fixed backbone).
\end{theorem}

\begin{proof}
We follow \cite{mah_bou_pap} for the proof. The details of both cases are derived in this paper. In the case of \textbf{n > (p+1)}, one may look at the proof in Theorem 3.1 therein. In the opposite case, one may look at the proof in Theorem 2.1 therein.
\end{proof}

\begin{proposition}
\label{ws_th}
For arbitrary distributions on countable spaces, the non-negative 2-Wasserstein Distance is unbounded from right.
\end{proposition}
\noindent \textit{Arguments.} Since we use a finite number of pairs of samples on which we fit empirical p.d.f., we work with the limits on optimal transport on discrete distributions. Following Thm. 2.1(b) of \cite{empirical_ot_pap}, it is clear that the best-known limit today is itself dependent on the choice of two data distributions \textbf{r} and \textbf{s}, and hence is not absolute. The non-negativeness of any distance is a trivial fact. In view of this, we work with the fact that 2-WD is unbounded from right, while bounded with 0 from left.
\begin{flushright}
\(\square\)
\end{flushright}

\begin{figure*}[t]
\setlength\tabcolsep{1pt}
\settowidth\rotheadsize{CrackForest}
\begin{center}
\begin{tabularx}{\linewidth}{l XXX}
\rothead{Xformer (SWiNv1)}   &   \includegraphics[width=\hsize,valign=m,trim=0 0 0 20, clip]{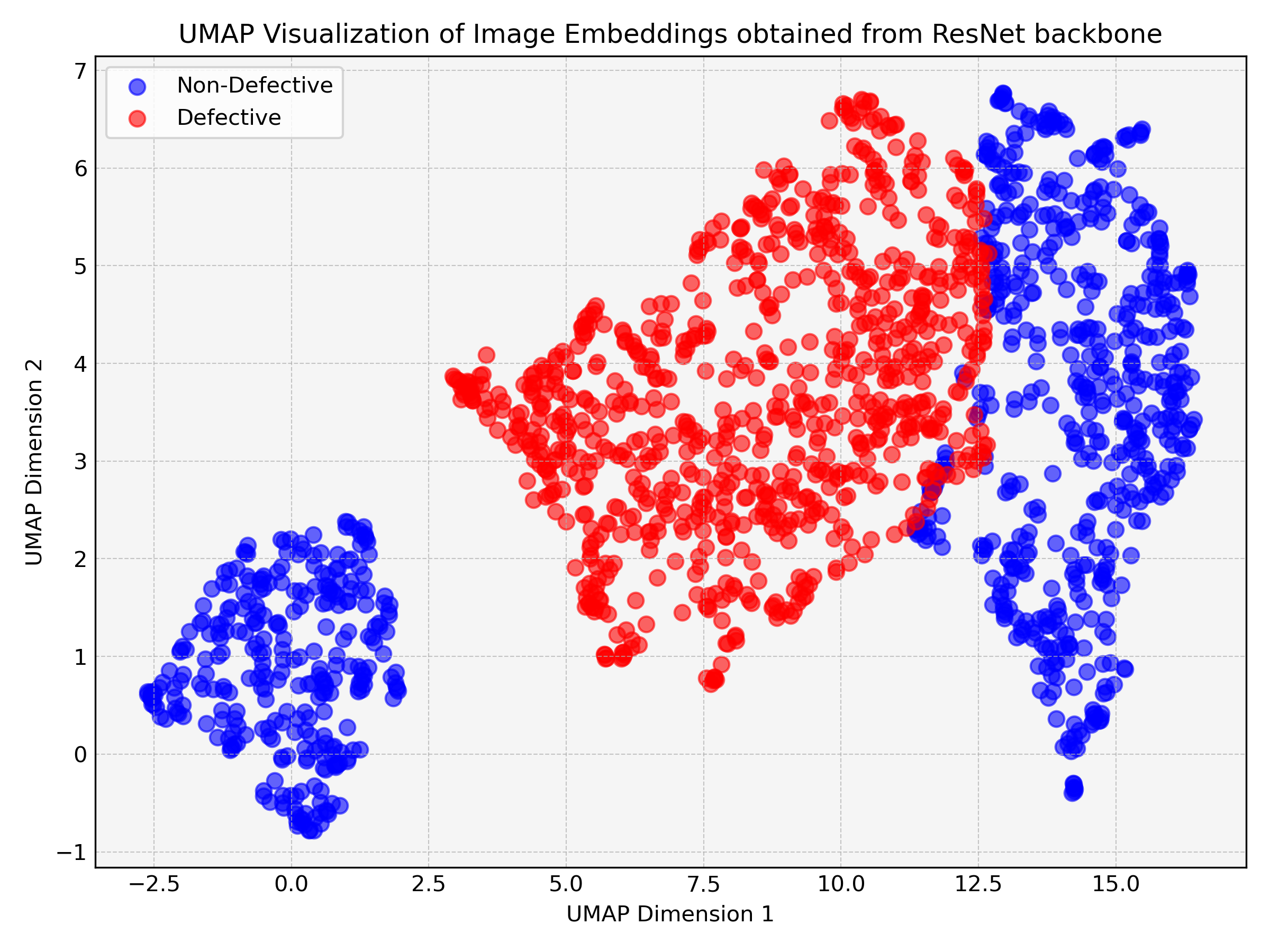}
                       &  \includegraphics[width=\hsize,valign=m,trim=0 0 0 20, clip]{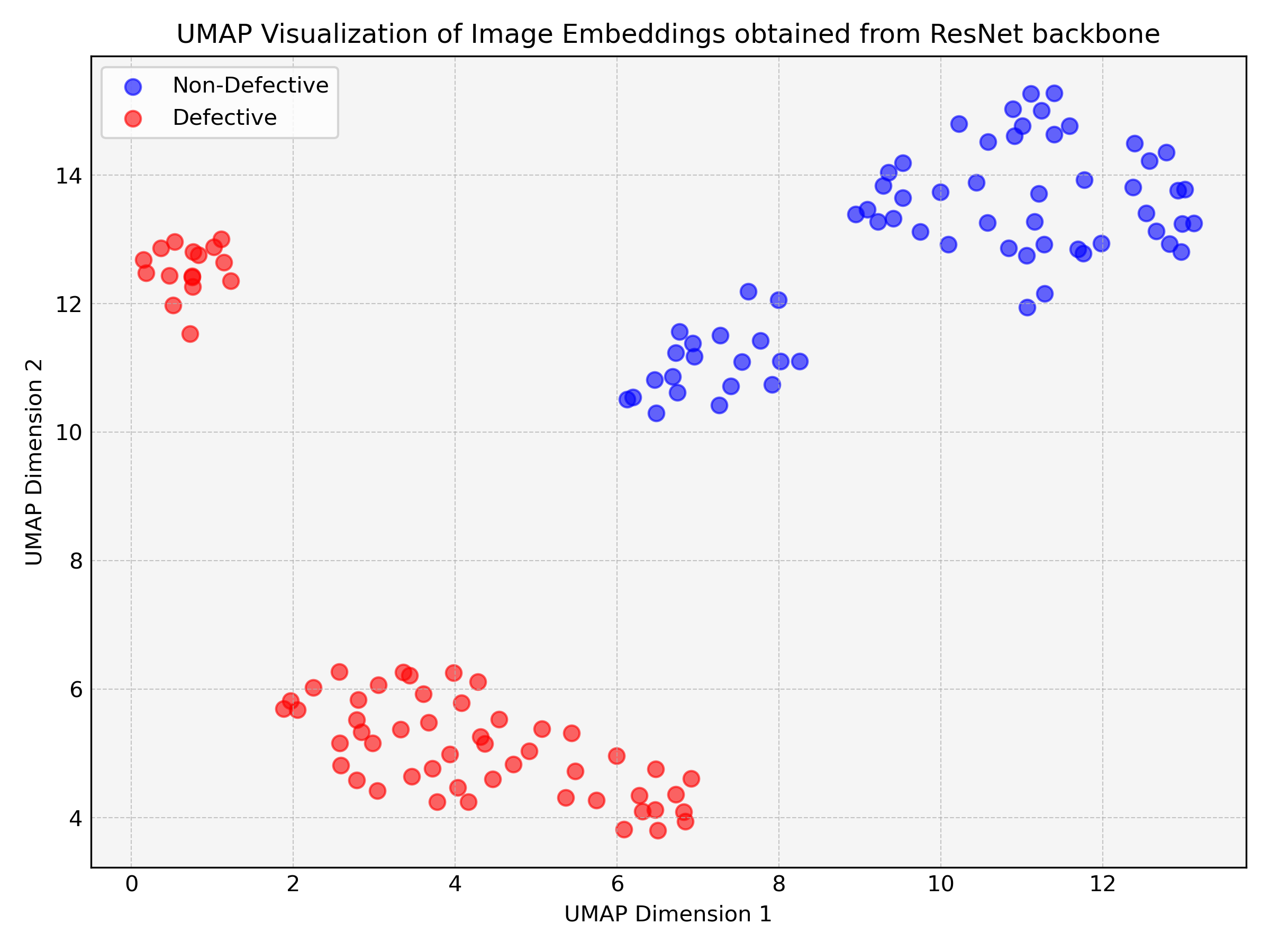} 
                       &  \includegraphics[width=\hsize,valign=m,trim=0 0 0 20, clip]{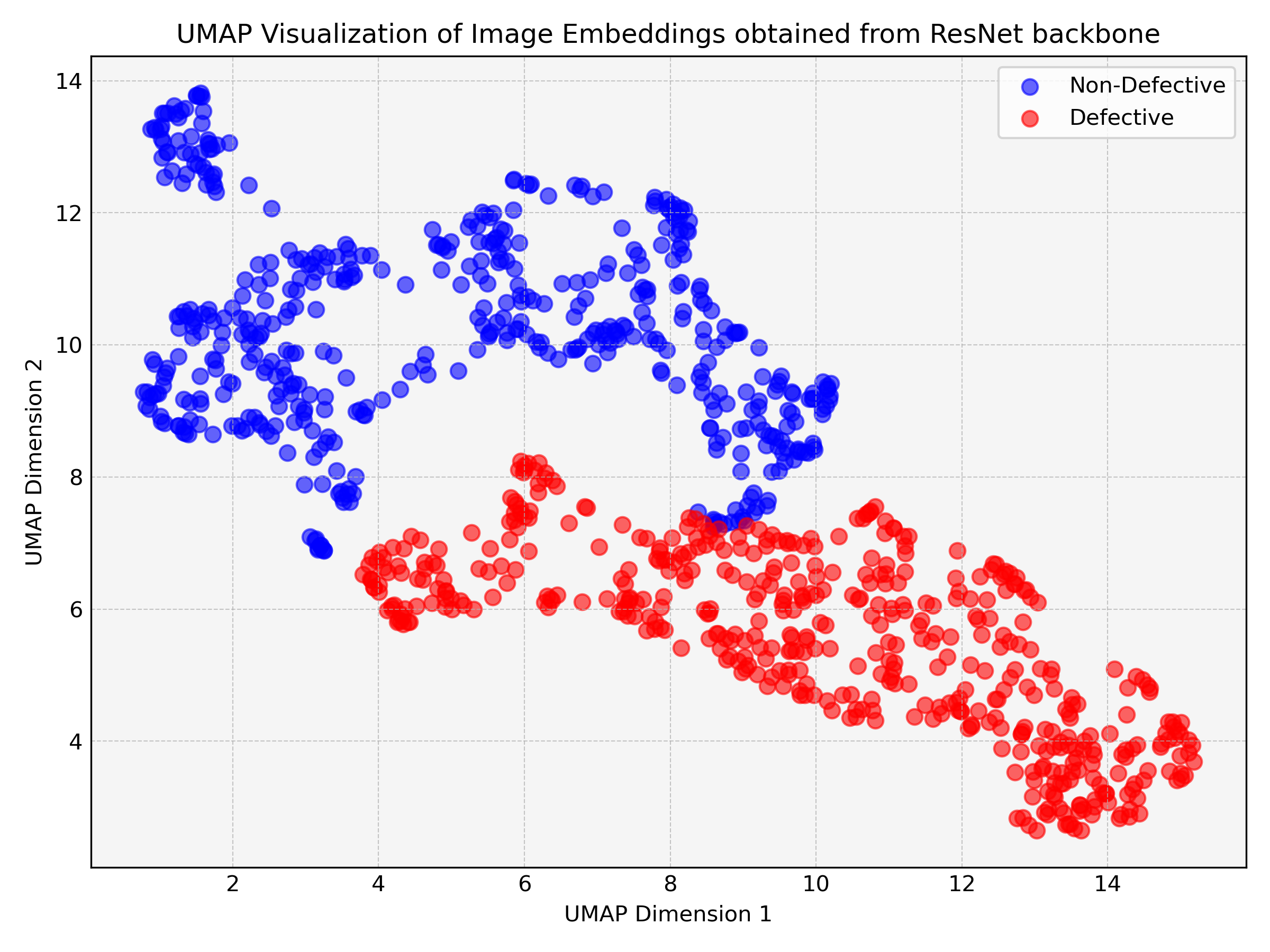}  \\  \addlinespace[4pt]
\rothead{CNN (DenseNet)}   &   \includegraphics[width=\hsize,valign=m,trim=0 0 0 20, clip]{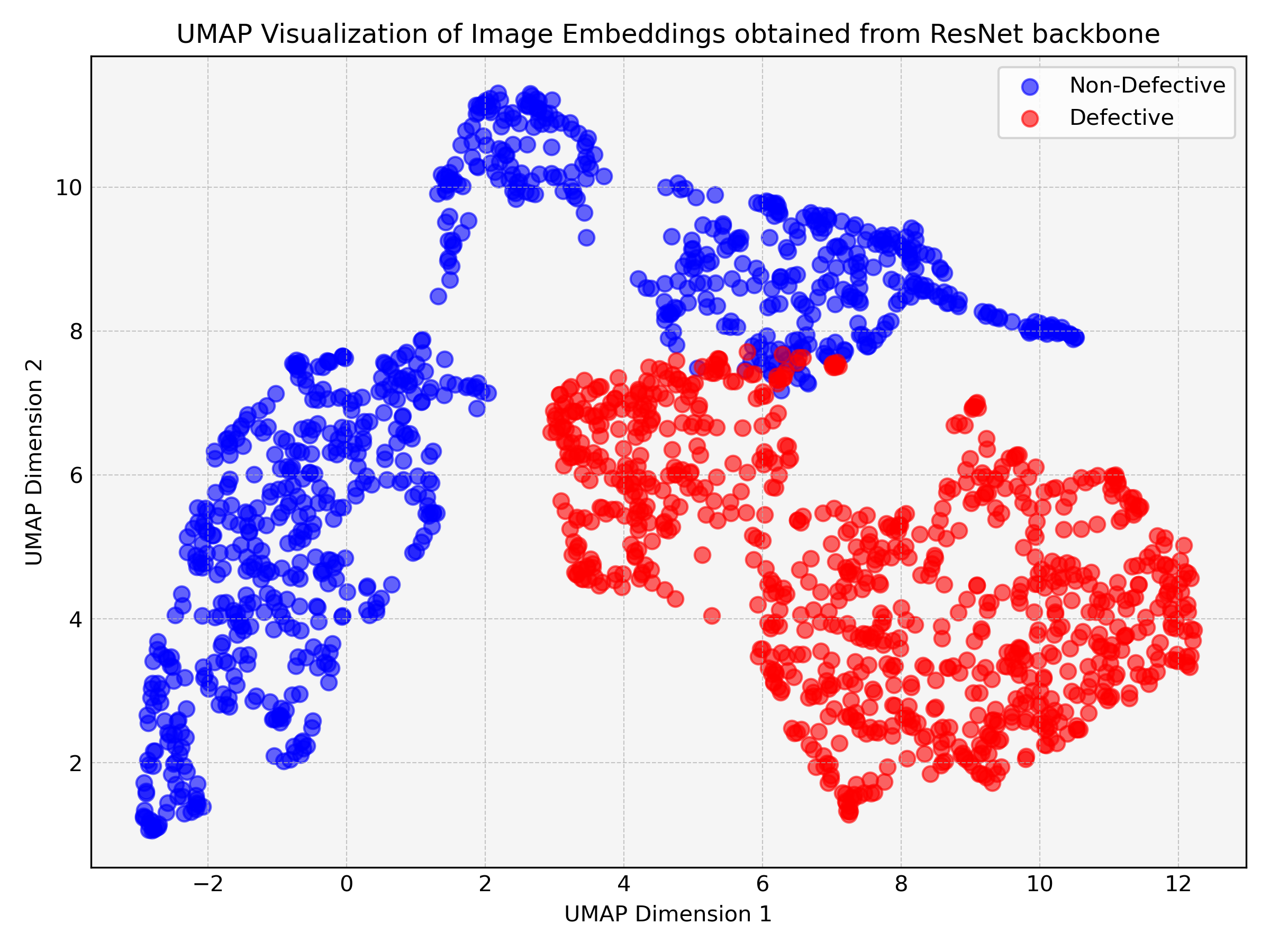}
                       &  \includegraphics[width=\hsize,valign=m,trim=0 0 0 20, clip]{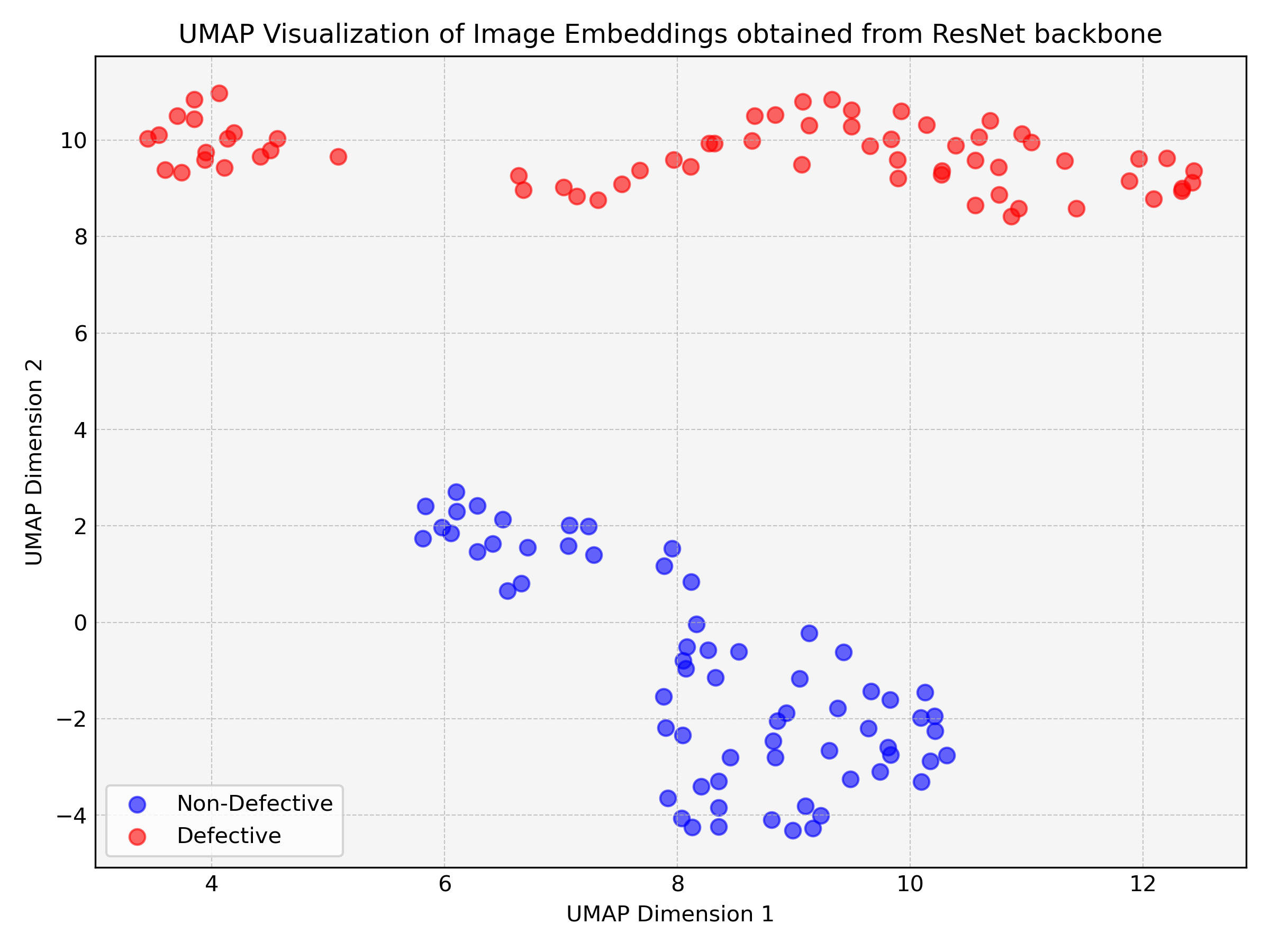} 
                       &  \includegraphics[width=\hsize,valign=m,trim=0 0 0 20, clip]{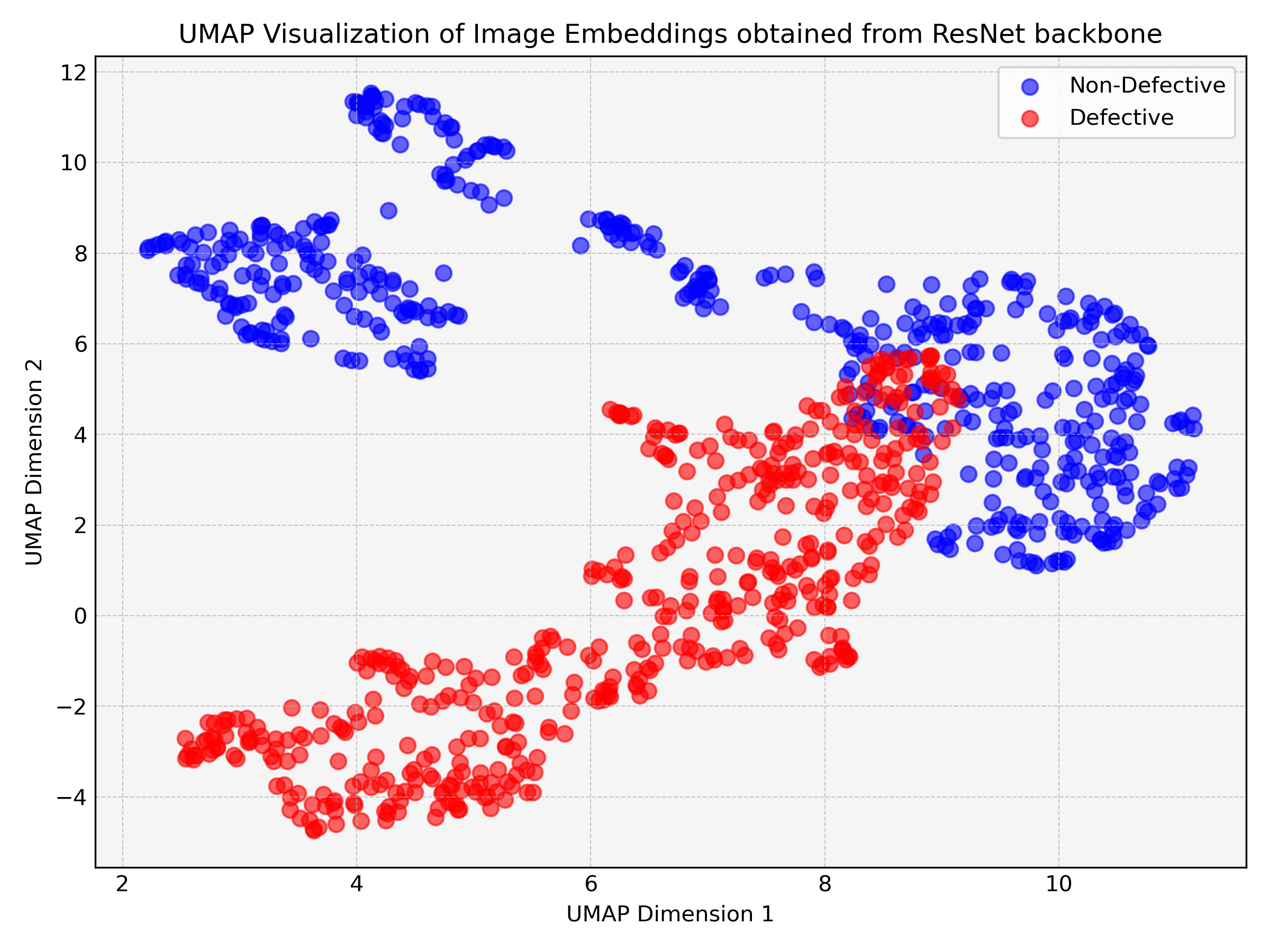}  \\  \addlinespace[4pt]
\rothead{Hybrid (PVTv2)}      &   \includegraphics[width=\hsize,valign=m,trim=0 0 0 20, clip]{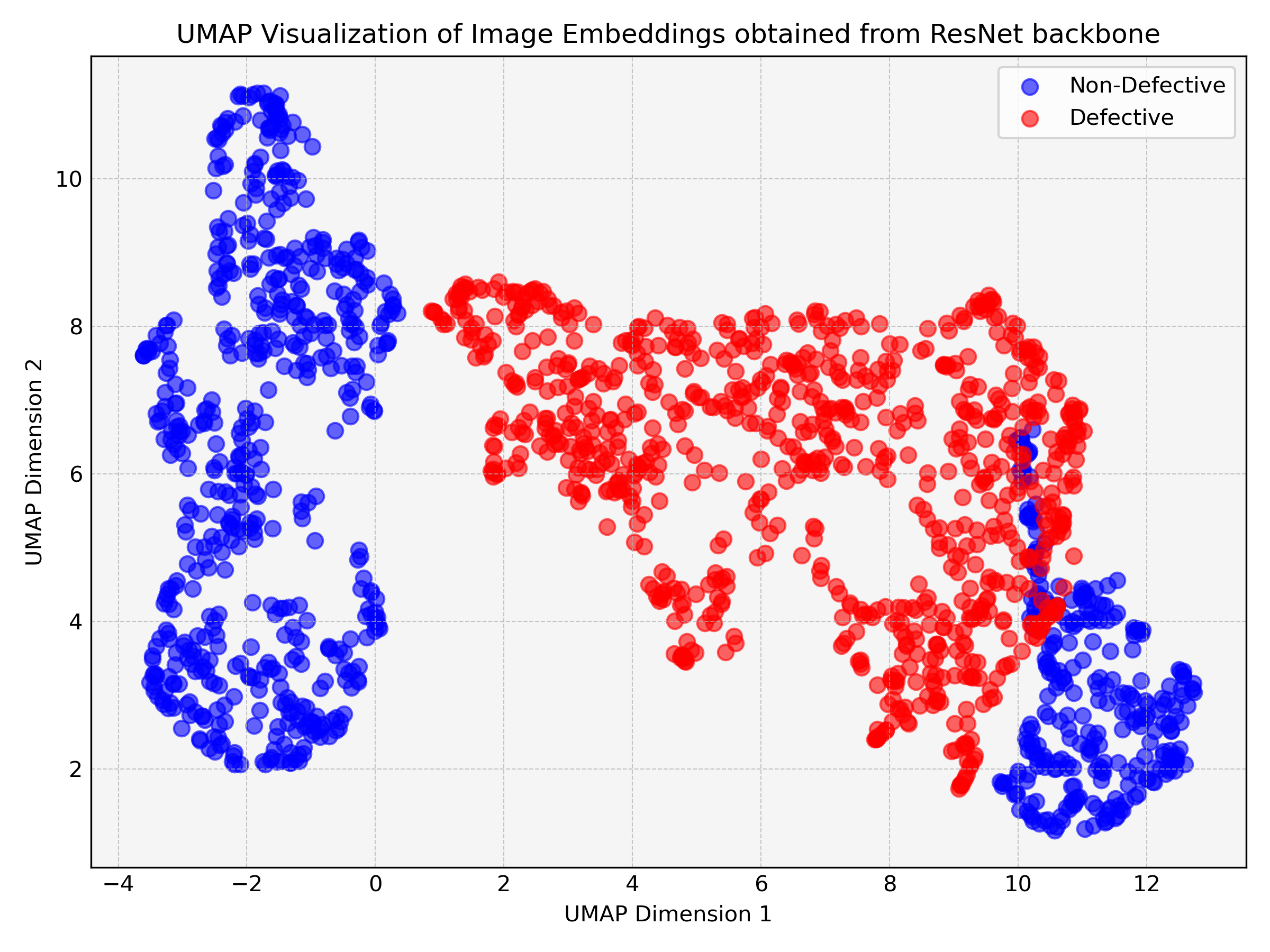}
                       &   \includegraphics[width=\hsize,valign=m,trim=0 0 0 20, clip]{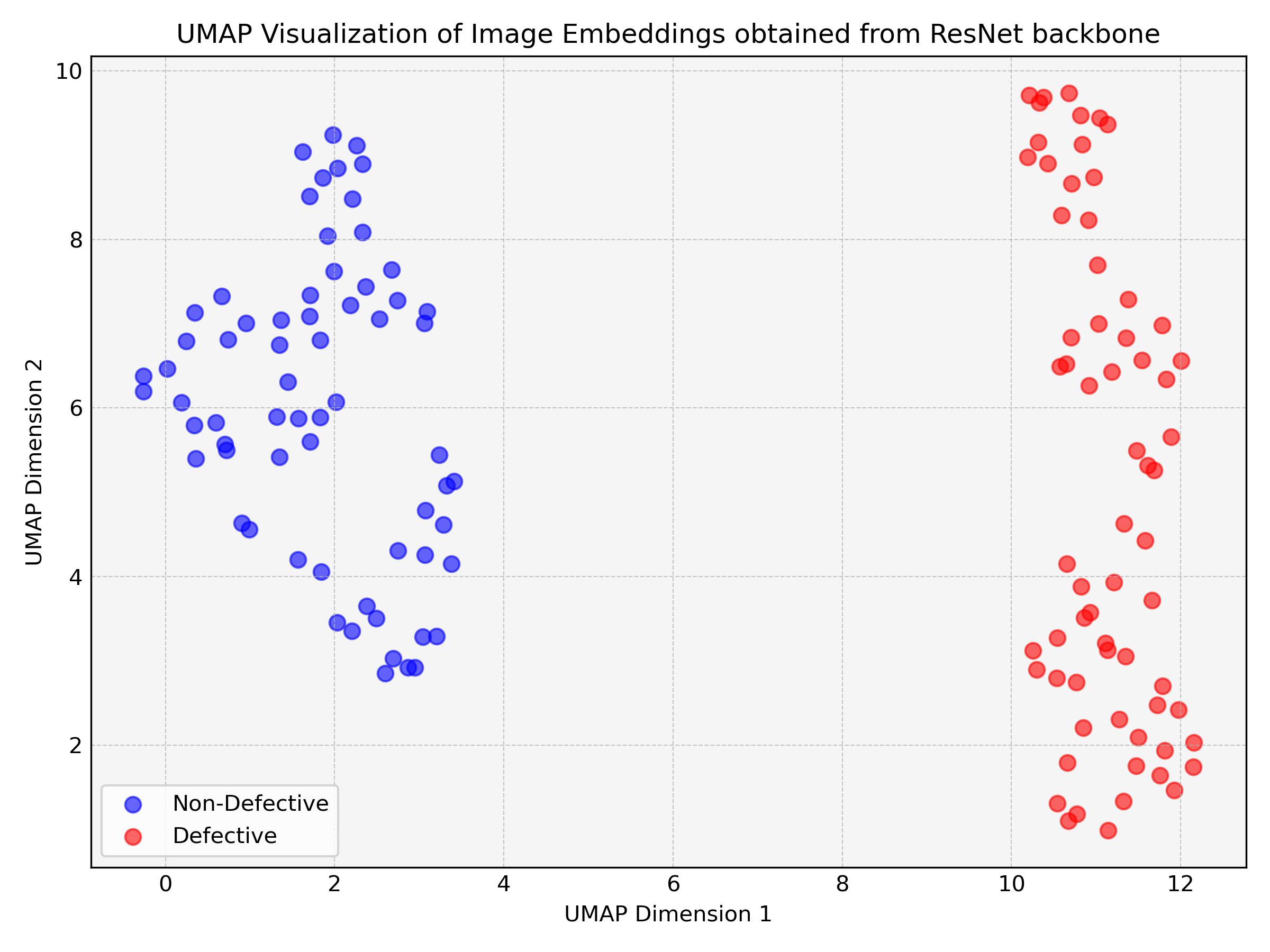}
                       &   \includegraphics[width=\hsize,valign=m,trim=0 0 0 20, clip]{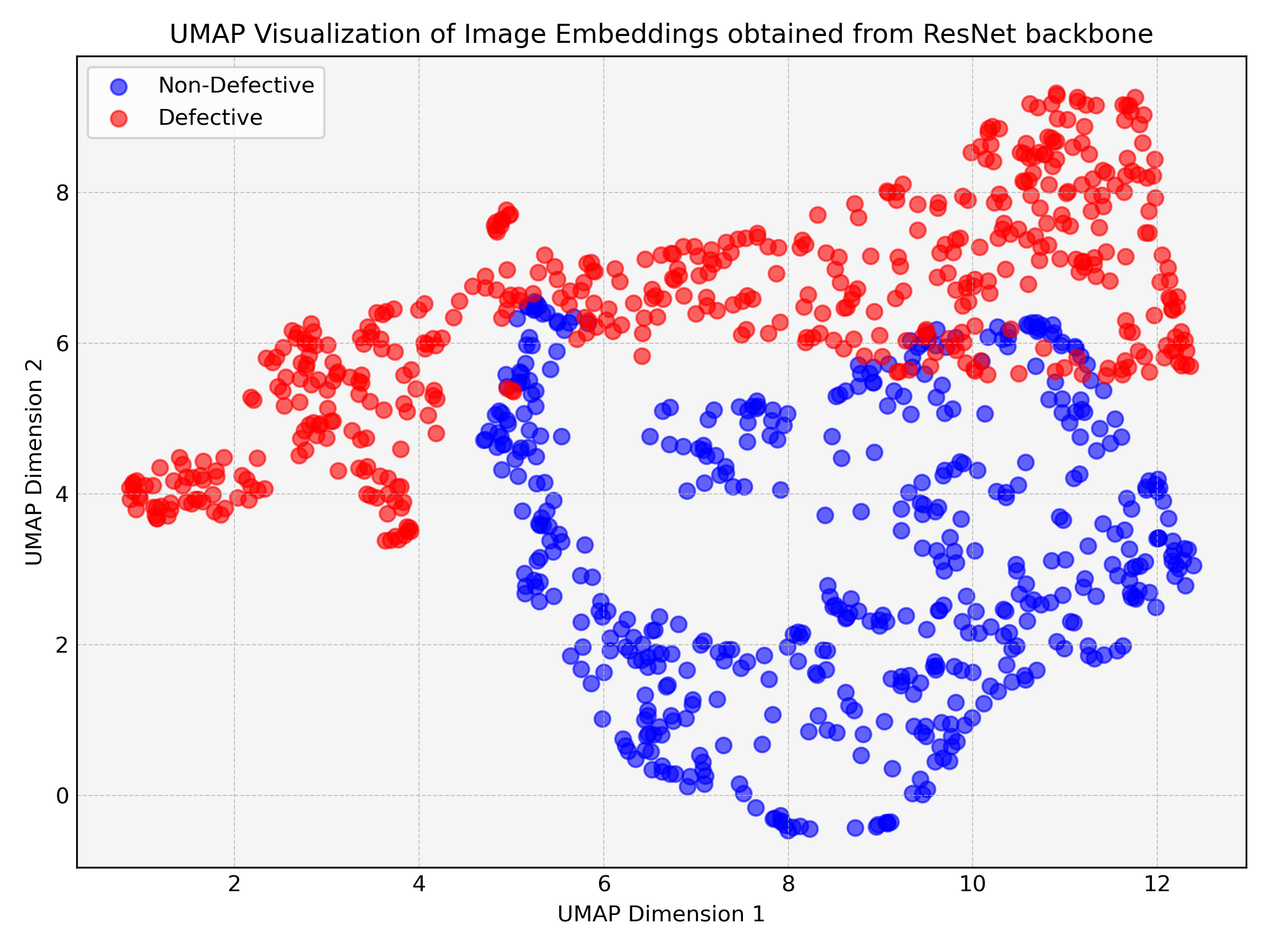} \\  \addlinespace[4pt]
                       & \centering $\langle$ NEU, Steel, Inclusion$\rangle$ & \centering $\langle$MVTec,Toothbrush, BadBristle$\rangle$ & \centering $\langle$BTAD,Ceramic Plate, Chipping$\rangle$
\end{tabularx}
\end{center}
\caption{Visualization of Cluster Relations Across Different types of backbones}
\label{umap_fig}
\end{figure*}

\begin{proposition}
\label{rmi_th}
Regionwise Mutual Information is unbounded from left, while bounded with +$\infty$ from right.
\end{proposition}
\noindent \textit{Arguments.} We follow some of the observations from \cite{rmi_ss_pap}. The required metric is $-\frac{1}{2}\mbox{log}\left(\mbox{det}(\mathbf{M})\right)$. Also, as explained, $\mbox{log}\left(\mbox{det}(\mathbf{M})\right)$ =$\sum_{i=1}^{d}\mbox{log}\lambda_i$. $\mathbf{M}$ being a positive-semidefinite (PSD) matrix, all $\lambda_i \geq 0$. Hence the \textit{loose} upper bound on RMI is +$\infty$. While there are certain upper bounds on determinants of PSD matrices known in literature, they are all based on the row sum/column sum/determinant of the matrix itself, i.e. no absolute limit. In view of this, we work with the fact that 2-WD is unbounded from left.
\begin{flushright}
\(\square\)
\end{flushright}

\subsection{Findings}
We first look at various \textbf{JS Divergences} for specific anomaly
classes and corresponding normal surface patterns. The results are shown as
points on \textcolor{red}{\textbf{red curve}} in Figs.~\ref{ds_1},
\ref{ds_2}, \ref{ds_3} and \ref{ds_4}. One can see that values are \textbf{very low}. Given the theoretical range (c.f. Thm. \ref{js_th}) $0 \leq JS \leq 1$, values are low for \textbf{all} anomaly classes. Table \ref{bb_js_ab_tab} shows some sample values. It is clear that these values are $<<$ 0.01, which \textbf{strongly agrees} with our \textbf{closeness} hypothesis.

Next, we look at various \textbf{Wasserstein Distances}. The results are
shown as points on the \textcolor{blue}{\textbf{blue curve}} in
Figs.~\ref{ds_1}, \ref{ds_2}, \ref{ds_3} and \ref{ds_4}. One can see that values are again \textbf{very low}. The low value, given the theoretical \textit{open} range $[0,)$, is yet again true for \textbf{all} anomaly classes. From certain samples values in Table ~\ref{bb_ws_ab_tab}, these values are generally in range of $0.3\leq WS \leq 2.0$, which again \textbf{strongly agrees} with our hypothesis.

Finally, we look at various \textbf{Mahalanobis Distances}. The results are
shown as points on the \textcolor{green}{\textbf{green curve}} in
Figs.~\ref{ds_1}, \ref{ds_2}, \ref{ds_3} and \ref{ds_4}. From theorem \ref{mb_th}, and knowing that the maximum number of class-specific anomaly samples is 497 for BTAD, 884 for NEU and 71 for MVTec respectively, knowing that the typical value of length of feature vector \textbf{p} is near 1000, the upper bound can be easily derived as 495.1 for BTAD, 882.1 for NEU and 69.2 for MVTec datasets. On the other hand, the mean value of distances are around 37.2 for BTAD, 102.7 for NEU and 28.3 for MVTec (c.f. Figs.~\ref{ds_1}, \ref{ds_2}, \ref{ds_3}). Hence these classwise-\textbf{mean MH} distances stand at $\sim$7.5\% of corresponding upper bound for BTAD, 11.6\% for NEU and 40.9\% for NEU. These relative figures once more \textit{highly agree} with our hypothesis.

The RMI values have a thoeretical range of (,+$\infty$]. The maximum class-specific value that we see in Fig.~\ref{rmi_fig} is around 300 (400 for MVTec). So even RMI values, though not conclusively supportive, are not outside the range.

\textit{Qualitatively}, we look at various UMAP visualizations in Figure \ref{umap_fig} for random anomaly classes from different datasets, using random backbones of selective categories. We can clearly and \textbf{consistently} see the trend of nearness for the clusters of embeddings for \textbf{all} the anomaly and normal classes. As a side observation, even though margin of separation of clusters is low, the overlap is not high, which is \textbf{very encouraging} since it proves that \textit{discriminative modeling}, wherever possible and done, has got good results (e.g. numerous results on cracks). The discrimination is so \textbf{robust} that in case of MVTec, we can observe cleanly separated sub-clusters within anomaly cluster, hinting at presence of multiple fine-grained defects.

%

\begin{table*}[t]
\caption[~]{JS Div. Variation across Backbones, Pre-training Datasets, Object and Anomaly Types. Entries are of order of $\mathbf{10^{-4}}$}
\label{bb_js_ab_tab}
\begin{center}
\begin{tabular}{|c|p{.08\linewidth}|c|c|c|c|c|c|c|c|c|c|c|c|} \hline \hline
\multicolumn{2}{|c|}{~} & \multicolumn{6}{c|}{Big Transfer (BiT)} & \multicolumn{6}{c|}{SWiNv1} \\ \cline{3-14}
\multicolumn{2}{|c|}{~} & \multicolumn{2}{c|}{50x1} & \multicolumn{2}{c|}{50x3} & \multicolumn{2}{c|}{101x3} & \multicolumn{2}{c|}{Tiny} & \multicolumn{2}{c|}{Small} & \multicolumn{2}{c|}{Base} \\ \cline{3-14} 
\multicolumn{2}{|c|}{~} & IN1K\tablefootnote{IN stands for ImageNet} & IN21K & IN1K & IN21K & IN1K & IN21K & IN1K & IN22K & IN1K & IN22K & IN1K & IN22K \\ \hline \hline
\multirow{2}{*}{NEU} & steel, patches\tablefootnote{The object type and object-specific defect types have been \textbf{randomly} chosen} & 24 & 11 & 7 & 2 & 11 & 9 & 37.9 & 37.8 & 38.3 & 38.2 & 30 & 30 \\ \cline{2-14}
 & steel, crazing & 45 & 19 & 10 & 5 & 18 & 16 & 38.1 & 37.8 & 38.2 & 38.1 & 30.0 & 30.1 \\ \hline
\multirow{2}{*}{MVTec} & wood, hole & 69 & 75 & 30 & 35 & 27 & 35 & 38.6 & 39.1 & 38.9 & 39.2 & 30.2 & 30.6 \\ \cline{2-14}
& grid, bent & 67 & 67 & 27 & 20 & 22 & 21 & 38.7 & 39.6 & 39.1 & 39.5 & 30.2 & 30.7 \\ \hline
\multirow{1}{*}{BTAD} & '03' & 40 & 63 & 21 & 20 & 22 & 23 & 38.5 & 38.3 & 38.6 & 38.3 & 30.0 & 30.2 \\ \hline  \hline
\end{tabular}
\end{center}
\end{table*}

\begin{table*}[t]
\caption{Maha. Dist. Variation across Backbones, Pre-training Datasets, Obj. and Anomaly Types}
\label{bb_mh_ab_tab}
\begin{center}
\begin{tabular}{|c|p{.08\linewidth}|c|c|c|c|c|c|c|c|c|c|c|c|} \hline \hline
\multicolumn{2}{|c|}{~} & \multicolumn{6}{c|}{Big Transfer (BiT)} & \multicolumn{6}{c|}{SWiNv1} \\ \cline{3-14}
\multicolumn{2}{|c|}{~} & \multicolumn{2}{c|}{50x1} & \multicolumn{2}{c|}{50x3} & \multicolumn{2}{c|}{101x3} & \multicolumn{2}{c|}{Tiny} & \multicolumn{2}{c|}{Small} & \multicolumn{2}{c|}{Base} \\ \cline{3-14} 
\multicolumn{2}{|c|}{~} & IN1K & IN21K & IN1K & IN21K & IN1K & IN21K & IN1K & IN22K & IN1K & IN22K & IN1K & IN22K \\ \hline \hline
\multirow{2}{*}{NEU} & steel, patches & 162.3 & 179.6 & 316.7 & 318.5 & 332.4 & 302.8 & 164.0 & 190.2 & 120.9 & 132.1 & 158.4 & 181.6 \\ \cline{2-14}
 & steel, crazing & 197.5 & 206.2 & 323.7 & 391.5 & 365.5 & 358.4 & 374.1 & 697.3 & 216.2 & 213.8 & 360.1 & 862.6 \\ \hline
\multirow{2}{*}{MVTec} & wood, hole & 46.50 & 44.1 & 90.93& 89.20 & 103.3 & 99.61 & 11.53 & 10.11 & 11.81 & 7.31 & 7.46 & 7.64 \\ \cline{2-14}
& grid, bent & 64.06 & 78.56 & 99.18 & 145.6 & 107.9 & 145.0 & 20.80 & 12.10 & 24.32 & 9.71 & 8.15 & 9.00 \\ \hline
\multirow{1}{*}{BTAD} & '03' & 116.4 & 141.6 & 221.0 & 287.9 & 210.5 & 268.4 & 78.58 & 56.70 & 61.93 & 46.50 & 53.30 & 53.29 \\ \hline  \hline
\end{tabular}
\end{center}
\end{table*}

\subsection{Effect of Varying Backbones}
As another experiment, we varied the backbone model sizes and pretraining dataset sizes, to understand till what degree the choice of backbone matters. The results of randomly chosen anomaly classes are shown in Tables \ref{bb_js_ab_tab}, \ref{bb_ws_ab_tab} and \ref{bb_mh_ab_tab}. We can observe that as the backbone size increases, the distances decrease in general. Across all anomaly classes, we saw this trend around 62\% of the time. When architecture is fixed but the pretraining dataset size is increased, the distance consistently decreases, with rare exception. The cluster overlaps still remain encouraging, as we saw in Fig.~\ref{umap_fig}. We can conclude that the pretraining dataset size matters a lot in choice of backbone.

\begin{table*}[t]
\caption{Wass. Dist. Variation across Backbones, Pre-training Datasets, Obj. and Anomaly Types}
\label{bb_ws_ab_tab}
\begin{center}
\begin{tabular}{|c|p{.08\linewidth}|c|c|c|c|c|c|c|c|c|c|c|c|} \hline \hline
\multicolumn{2}{|c|}{~} & \multicolumn{6}{c|}{Big Transfer (BiT)} & \multicolumn{6}{c|}{SWiNv1} \\ \cline{3-14}
\multicolumn{2}{|c|}{~} & \multicolumn{2}{c|}{50x1} & \multicolumn{2}{c|}{50x3} & \multicolumn{2}{c|}{101x3} & \multicolumn{2}{c|}{Tiny} & \multicolumn{2}{c|}{Small} & \multicolumn{2}{c|}{Base} \\ \cline{3-14} 
\multicolumn{2}{|c|}{~} & IN1K & IN21K & IN1K & IN21K & IN1K & IN21K & IN1K & IN22K & IN1K & IN22K & IN1K & IN22K \\ \hline \hline
\multirow{2}{*}{NEU} & steel, patches & .5196 & .4090 & .6054 & .5029 & .6979 & .51936 & .5537 & .5659 & .7102 & .5113 & .7949& .7884 \\ \cline{2-14}
 & steel, crazing & .6740 & .4659 & .6844 & .5783 & .7639 & .5196  & .6950 &.7643 & .8397 & .5786 & .9958 & 1.184 \\ \hline
 
\multirow{2}{*}{MVTec} & wood, hole & 1.275 & 1.411 & 1.265 & 1.376 & 1.229 & 1.290 & 1.202 & 1.539 & 1.449 & 1.847 & 1.685 & 1.868 \\ \cline{2-14}
& grid, bent & 1.030 & 1.042 & 1.041 & 1.172 & .9998 & 1.094 & 1.214 & 1.730 & 1.722 & 1.856 & 1.775 & 1.815 \\ \hline
\multirow{1}{*}{BTAD} & '03' & .58 & .56 & .76 & .73 & .79 & .74 & .938 & .915 & 1.20 & .926 & .869 & 1.06 \\ \hline  \hline
\end{tabular}
\end{center}
\end{table*}

\begin{table*}[h]
\caption{Measurements using Backbones trained on \textbf{Non-Imagenet} Datasets}
\label{non_imagenet_tab}
\begin{center}
\begin{tabular}{|p{.08\linewidth}|c|c|c|c|c|c|c|c|c|} \hline \hline
\multicolumn{1}{|c|}{~} & \multicolumn{3}{c|}{Dinov2} & \multicolumn{3}{c|}{OpenCLIP} & \multicolumn{3}{c|}{SWAG} \\ 
    \multicolumn{1}{|c|}{~} & \multicolumn{3}{c|}{trained using \textbf{LVD-142M}} & \multicolumn{3}{c|}{trained using \textbf{LAION-2B}} & \multicolumn{3}{c|}{trained using \textbf{IG3.6B}} \\  \cline{2-10}
\multicolumn{1}{|c|}{~} & JS Div & Maha. Dist & Wass. Dist & JS Div & Maha. Dist & Wass. Dist & JS Div & Maha. Dist & Wass. Dist \\ \hline \hline
 NEU & 0.00026 & 7.67283 & 0.03414 & 0.00454 & 34.99813 & 0.13339 & 0.00375 & 40.03888 & 0.20525 \\ \hline
MVTec & 0.00991 & 15.82391 & 0.93268 & 0.00488 & 27.13522 & 0.58536 & 0.00398 & 30.58091 & 1.25381 \\ \hline
BTAD & 0.00977 & 58.58314 & 0.97567 & 0.00477 & 43.01020 & 0.43319 & 0.00387 & 60.68180 & 0.84217 \\ \hline
CB & 0.00226 & 42.53444 & 0.06645 & 0.00466 & 32.53861 & 0.29289 & 0.00384 & 34.28002 & 0.5946 \\ \hline  \hline
\end{tabular}
\end{center}
\end{table*}

\subsubsection{Effect of Changing Backbone Dataset}
As is known \cite{bbone_battle_pap}, variants of ImageNet dataset have remained at the core of training of most backbones. However, in last few years, continuously bigger datasets have arisen, and sporadic backbones have also been published using them. To \textit{avoid bias propagation} from ImageNet dataset reflecting into our hypothesis, we show in Table \ref{non_imagenet_tab} that for backbones trained on other datasets also, the distances and divergences remain quite low. This reaffirms the fact that our hypothesis is quite generic.

\subsection{Statistical Significance of Hypothesis}
To conclude our analysis, we also undertake our hypothesis through significance testing.
Statistical significance test is carried out to compute the probability (\textbf{p-value}) and establish that the relationship between two variables is not a chance event. The random variables in question here relate to anomalous and anomaly-free FG/BG patterns. Since our measurements of the random variables are in form of distances and divergences, we define the \textbf{null hypothesis $H_0$} as the measurements being \textbf{of-same-order/comparable} when taken between anomalous and anomaly-free samples, vis-a-vis when taken between anomalous samples and samples from an uncorrelated domain (e.g. object background). The \textbf{alternative hypothesis $H_1$} hence tantamounts to measurements being \textbf{not-of-same-order/incomparable} when taken between anomalous and anomaly-free samples, vis-a-vis when taken between anomalous samples and samples from object background. We choose object background for the unrelated domain, since sampling regions from it is quite easy.

\begin{figure*}[h]
\begin{center}
\includegraphics[trim=80 40 100 80,clip,width=.46\textwidth]{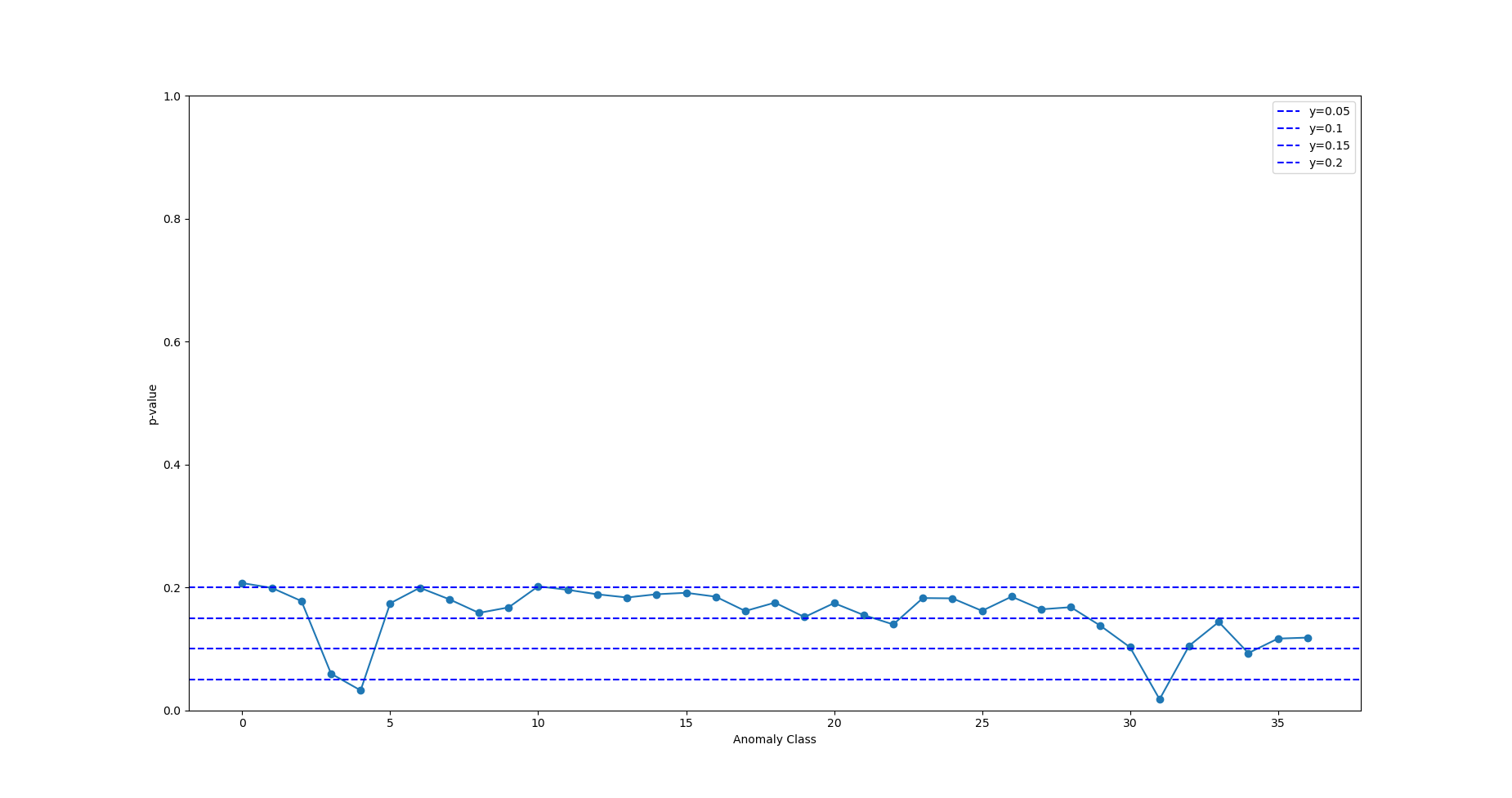}
\caption{P-values for various anomaly classes, object foregrounds and backgrounds in MVTec}
\label{pval_fig}
\end{center}
\end{figure*}

Since the population data is neither nominal nor ordinal, we employ \textbf{one-tailed} t-test instead of Chi-square test. We carry out a two-sample t-test rather than a paired t-test\footnote{Hence we involve samples from object background}, since we cannot generate paired data of distance measurements for our experiment. We consider the \textit{homoscedastic} case. In section \ref{dataprep_sec}, we had introduced two sample sets $\mathrm{D_{r,i,c,s}}$ and $\mathrm{N_{r,i,o_{c},s}}$, and the first group of measurements, $\mathbf{D_{A,FG}}$ was done between these sets. We can similarly define foreground-paired background sample set as $\mathrm{N_{r,i,bg_{c},s}}$, and have another group of measurements $\mathbf{D_{A,BG}}$ between $\mathrm{D_{r,i,c,s}}$ and $\mathrm{N_{r,i,bg_{c},s}}$. Assuming that the number of samples/cropped region are all same across the three sets, the two groups of measurements will also have same cardinality i.e. $n_1$ $=$ $n_2$. Hence, we first calculate the t-statistic\footnote{given low amount of anomaly samples, we employ Student t-statistic instead of Z-statistic} as follows
\begin{equation}
    t~=~\frac{\mu_{D_{A,FG}}-\mu_{D_{A,BG}}}{s_p*\left(\sqrt{\frac{1}{n_1}+\frac{1}{n_2}}\right) }
\end{equation}
\noindent where
\begin{equation}
    s_p~=~\sqrt{\frac{\left(n_1-1\right)*\sigma_{D_{A,FG}}^2+\left(n_2-1\right)\sigma_{D_{A,BG}}^2}{\left(n_1+n_2-2\right)}}
\end{equation}
Here, $\mu$ and $\sigma$ denote the mean and variance of each group of measurements. Given the t-score thus computed, and degrees of freedom $df~=~\left(n_1+n_2-2\right)$, we use the standard t-distribution tables \cite{stat_tab_book} to look up the p-value as the column head.

For MVTec dataset, p-values for various anomaly classes is shown in Fig~\ref{pval_fig}. The distance groups are based on shortlisted backbones which give maximum separation. As can be seen, the values mostly fall in the range of 0.15 and 0.20. Since these values are not smaller than the popular alpha value of 0.05, if one considers statistics alone, strictly rejecting null hypothesis may not be possible. However, our population data is quite small (few tens of samples) since anomalies arise scarcely in practice. Hence, at a low level of $\sim$0.15, given the compelling experimental evidences earlier, we still consider our hypothesis to be significant.

\begin{table*}[h]
\caption{Image-level Transfer Learning Performance using Comparable Datasets}
\label{tl_img_tab}
\begin{center}
\begin{adjustbox}{width={\textwidth},totalheight={\textheight},keepaspectratio,}%
\begin{tabular}{|c|c|c|c|c|c|c|c|c|c|c|c|c|c|} \hline \hline
 & bottle & cable & carpet & hazelnut & leather & metal\_nut & pill & screw & tile & toothbrush & transistor & wood & zipper \\ \hline \hline
ImageNet-15 & 92.87          & 56.41          & 48.8           & 94.03          & \textbf{80.94} & 42.46          & 68.75          & 64.33         & \textbf{90.45} & 67.5           & 60.43          & \textbf{95.77} & \textbf{70.51} \\ \hline
Mvtec-15    & \textbf{93.58} & \textbf{56.99} & \textbf{52.84} & \textbf{95.69} & 79.14          & \textbf{42.59} & \textbf{69.29} & \textbf{68.3} & 88.81          & \textbf{67.75} & \textbf{62.76} & 95.06          & 68.24 \\ \hline \hline
\end{tabular}
\end{adjustbox}
\end{center}
\end{table*}

\begin{table*}[h]
\caption{Pixel-level Transfer Learning Performance using Comparable Datasets}
\label{tl_pix_tab}
\begin{center}
\begin{adjustbox}{width={\textwidth},totalheight={\textheight},keepaspectratio,}%
\begin{tabular}{|c|c|c|c|c|c|c|c|c|c|c|c|c|c|} \hline \hline
 & bottle & cable & carpet & hazelnut & leather & metal\_nut & pill & screw & tile & toothbrush & transistor & wood & zipper \\ \hline \hline
ImageNet-15 & 88.36          & \textbf{89.47} & \textbf{73.03} & 97.78          & 90.26          & 83.63         & \textbf{93.82} & 93.66          & \textbf{86.53} & \textbf{92.79} & \textbf{82} & 82.59          & 74.05 \\ \hline
Mvtec-15    & \textbf{90.86} & 89.34          & 70.95          & \textbf{97.98} & \textbf{91.66} & \textbf{88.7} & 93.09          & \textbf{94.11} & 83.39          & 91.81          & 79.51       & \textbf{82.84} & \textbf{84.45} \\ \hline \hline
\end{tabular}
\end{adjustbox}
\end{center}
\end{table*}

\section{Employing the Hypothesis in Real-life AD}
\label{tl_sec}
In this section, we give \textit{first} proof that our hypothesis indeed
has practical, implementational advantage. Especially for domain adaptation
(DA) though not for anomaly detection, there are many works that point to
the fact that using a source domain that is nearby to target domain(e.g.
sketches in relation with paintings) give improved performance(e.g.,
\cite{d3t_gan_pap}, \cite{dwsc_pap}). Motivated with this, we tried to
carry out similar experiment in context of few-shot anomaly detection using
transfer learning, a \textit{generalization} of DA. Specifically, we use a
very recent and popular model, RegAD \cite{reg_ad_pap} as baseline, and try
to replace the inner ImageNet-pretrained ResNet-18 backbone with a backbone
trained from anomaly-free patches of objects from MVTec dataset. Other than
matching the data, e.g. 15 classes only, 1300 samples per class, we do
\textbf{no other change}, especially in model hyperparameters. We report
results on a majority of classes within MVTec, \textbf{without any
hyperparameter tuning}. As can be seen from tables \ref{tl_img_tab} and \ref{tl_pix_tab}, the \textit{extended} model
\textit{surpasses} the baseline for most classes, and the performance gap
for remaining classes is narrow. The gap is quite pronounced for the image-level anomaly detection task. It is obvious that with more
careful integration of our hypothesis, the performance of modified RegAD
model will indeed establish a new baseline. We expect similar improved
performance over other tasks and models as well.

\section{Possible Impacts and Conclusion}
In this paper, we have claimed an important hypothesis that the space of anomaly-free visual patterns of the normal samples correlates well with each of the various spaces of patterns of the class-specific anomaly samples. We worked with latent representations of class-specific visually anomalous regions and normal regions, using
various backbones. We varied the anomaly datasets to cover both indoor and outdoor inspections, \textbf{70+} anomaly classes, employed \textbf{50+} popular backbones of various types which are trained over diverse source datasets, employ various image-space and latent-space metrics to exhaustively, experimentally verify the hypothesis both qualitatively and quantitatively. Other than very few outliers, we find this hypothesis to be generally true. We also found that the corresponding clusters, even when close, had very low overlaps. 
The experimental results were further formally verified using statistical hypothesis testing that employs Student's t-distribution.
Going a step further, we have unambiguously established the first
encouraging results of using our hypothesis, without any bells and
whistles, to establish a new AD baseline, by extending RegNet. Hence we expect that our finding of a simple closeby domain that readily entails a large number of samples, and which also oftentimes shows interclass separability though with narrow margins, will be a useful discovery. Especially, it is expected to improve domain adaptation for anomaly detection, and few-shot learning for anomaly detection, nudging the future AD solutions to truly approach in-the-wild settings.

\bibliography{article}


\end{document}